\documentclass[11pt]{article}

\usepackage{enumerate}
\usepackage{enumitem}
\usepackage{bm}
\usepackage{natbib}
\usepackage{mathtools}
\usepackage{amsfonts}
\usepackage[stable]{footmisc}
\usepackage{floatflt}
\usepackage{wrapfig}
\usepackage{algorithm}
\usepackage{algorithmic}
\usepackage{extarrows}
\usepackage{pbox}
\usepackage{framed}
\usepackage{wrapfig}
\usepackage{footnote}
\usepackage{amsthm}
\usepackage{times}
\usepackage{fullpage}
\allowdisplaybreaks

\newcommand{\e}{{\bf e}}
\newcommand{\p}{{\bf p}}
\newcommand{\q}{{\bf q}}

\newcommand{\z}{{\bf z}}

\newcommand{\hl}{{\hat{\ell}}}
\newcommand{\tL}{{\tilde{L}}}
\newcommand{\TL}{{\bf\tL}}

\newcommand{\I}{\mathcal{I}}

\newcommand{\G}{{\mathcal{G}}}

\newcommand{\Y}{{\mathcal{Y}}}

\newcommand{\R}{{\bf R}}
\newcommand{\RE}{{\text{\rm RE}}}
\newcommand{\C}{\mathbb{C}}
\newcommand{\E}{\mathbb{E}}
\newcommand{\1}{{\bm 1}}

\newcommand{\ANH}{{AdaNormalHedge}}
\newcommand{\ANHTV}{{AdaNormalHedge.TV}}

\newcommand{\specialcell}[2][c]{\begin{tabular}[#1]{@{}c@{}}#2\end{tabular}}

\renewcommand{\O}{\hat{O}}
\renewcommand{\C}{{\bf C}}
\renewcommand{\L}{{\bf L}}
\renewcommand{\u}{{\bf u}}
\renewcommand{\r}{{\bf r}}

\renewcommand{\P}{{\mathcal{P}}}

\renewcommand{\l}{{\bm \ell}}
\renewcommand{\(}{\left(}
\renewcommand{\)}{\right)}

\makesavenoteenv{tabular}
\makesavenoteenv{table} 
\makesavenoteenv{algorithm} 
\makesavenoteenv{algorithmic} 

\newtheorem{theorem}{Theorem}
\newtheorem{lemma}{Lemma}

\title{Achieving All with No Parameters: Adaptive NormalHedge}

\author{Haipeng Luo \\ Princeton University \\ haipengl@cs.princeton.edu
\and
Robert E. Schapire \\ Microsoft Research and Princeton University \\ schapire@cs.princeton.edu
}
\begin{document}

\maketitle

\begin{abstract}
We study the classic online learning problem of predicting with expert advice,
and propose a truly parameter-free and adaptive algorithm that
achieves several objectives simultaneously without using any prior information.
The main component of this work is an improved version of 
the NormalHedge.DT algorithm \citep{LuoSc14b}, called AdaNormalHedge.
On one hand, this new algorithm ensures small regret when the competitor has small loss
and almost constant regret when the losses are stochastic.
On the other hand, the algorithm is able to compete with any convex combination of the experts simultaneously,
with a regret in terms of the relative entropy of the prior and the competitor.
This resolves an open problem proposed by \citet{ChaudhuriFrHs09} and \citet{ChernovVo10}.
Moreover, we extend the results to the sleeping expert setting
and provide two applications to illustrate the power of AdaNormalHedge:
1) competing with time-varying unknown competitors
and 2) predicting almost as well as the best pruning tree.
Our results on these applications significantly improve previous work 
from different aspects,
and a special case of the first application resolves another open problem proposed by 
\citet{WarmuthKo14}
on whether one can simultaneously achieve optimal shifting regret 
for both adversarial and stochastic losses.%\footnote{This paper is eligible for a student award.}
\end{abstract}

%\begin{keywords}
%expert algorithm, NormalHedge, adaptivity, unknown competitors, time-varying competitors, 
%first order bounds, sleeping expert, adaptive regret, shifting regret
%\end{keywords}

\section{Introduction}
The problem of predicting with expert advice was first pioneered by 
\citet{LittlestoneWa94, FreundSc97, CesabianchiFrHeHaScWa97, Vovk98} and others two decades ago.
Roughly speaking, in this problem,
a player needs to decide a distribution over a set of experts on each round,
and then an adversary decides and reveals the loss for each expert.
The player's loss for this round is the expected loss of the experts with respect to
the distribution that he chose,
and his goal is to have a total loss that is not much worse than any single expert,
or more generally, any fixed and unknown convex combination of experts.

%we address the problem of catching any unknown competitor with zero prior information,
%and develop practical algorithms that are truly parameter-free and adaptive to the environment.

Beyond this classic goal, various more difficult objectives for this problem were studied in recent years, such as:
learning with unknown number of experts and 
competing with all but the top small fraction of experts \citep{ChaudhuriFrHs09, ChernovVo10};
competing with a sequence of different combinations of the experts \citep{HerbsterWa01, CesabianchiGaLuSt12};
learning with experts who provide confidence-rated advice \citep{BlumMa07};
and achieving much smaller regret when the problem is ``easy'' while still ensuring worst-case robustness
\citep{DeRooijErGrKo14, VanErvenKoWa14, GaillardStEr14}.
Different algorithms were proposed separately to solve these problems to some extent.
In this work, we essentially provide {\it one single parameter-free algorithm} that achieves all these goals
with absolutely no prior information and significantly improved results in some cases.

Our algorithm is a variant of \citet{ChaudhuriFrHs09}'s NormalHedge algorithm,
and more specifically is an improved version of NormalHedge.DT \citep{LuoSc14b}.
We call it Adaptive NormalHedge (or {\ANH} for short).
NormalHedge and NormalHedge.DT 
provide guarantees for the so-called $\epsilon$-quantile regret simultaneously for any $\epsilon$,
which essentially corresponds to competing with a uniform distribution over the top $\epsilon$-fraction of experts.
Our new algorithm improves NormalHedge.DT from two aspects (Section \ref{sec:ANH}):
\begin{enumerate}
\item {\ANH} can compete with not just the competitor of the specific form mentioned above, 
but indeed any unknown fixed competitor simultaneously,
with a regret in terms of the relative entropy between the competitor and the player's prior belief of the experts.

\item
{\ANH} ensures a new regret bound in terms of
the {\it cumulative magnitude of the instantaneous regrets},
which is always at most the bound for NormalHedge.DT (or NormalHedge).
Moreover, the power of this new form of regret is almost the same as 
the second order bound introduced in a recent work by \citet{GaillardStEr14}.
Specifically, it implies 1) a small regret when the loss of the competitor is small
and 2) an almost constant regret when the losses are generated randomly
with a gap in expectation.
\end{enumerate}

Our results resolve the open problem asked in \citet{ChaudhuriFrHs09} and \citet{ChernovVo10}
on whether a better $\epsilon$-quantile regret in terms of the loss
of the expert instead of the horizon can be achieved.
In fact, our results are even better and more general.

{\ANH} is a simple and truly parameter-free algorithm. 
Indeed, it does not even need to know the number of experts in some sense.
To illustrate this idea, in Section~\ref{sec:sleeping}
we extend the algorithm and results to a setting where experts provide
confidence-rated advice \citep{BlumMa07}.
We then focus on a special case of this setting % where the confidence is either zero or one.
%In other words, on each round some experts abstain from making any advice. This is also 
called the sleeping expert problem \citep{Blum97, FreundScSiWa97}, %,and is a setting 
where the number of ``awake'' experts is dynamically changing
and the total number of underlying experts is indeed unknown.
{\ANH} is thus a very suitable algorithm for this problem.
To show the power of all the abovementioned properties of {\ANH},
we study the following two examples of the sleeping expert problem
and use {\ANH} to significantly improve previous work.

The first example is adaptive regret, that is, regret on any time interval,
introduced by \citet{HazanSe07}.
This can be reduced to a sleeping expert problem by adding a new copy
of each original expert on each round \citep{FreundScSiWa97, KoolenAdWa12}.
Thus, the total number of sleeping experts is not fixed.
When some information on this interval is known
(such as the length, the loss of the competitor on this interval, etc),
several algorithms achieve optimal regret \citep{HazanSe07, CesabianchiGaLuSt12}.
However, when no prior information is available, 
all previous work gives suboptimal bounds.
We apply {\ANH} to this problem.
The resulting algorithm, which we called {\ANHTV}, 
enjoys the optimal adaptive regret in not only the adversarial case
but also the stochastic case due to the properties of {\ANH}.

We then extend the results to the problem of tracking the best experts
where the player needs to compete with the best partition of the whole process
and the best experts on each of these partitions
\citep{HerbsterWa95, BousquetWa03}.
This resolves one of the open problems in \citet{WarmuthKo14}
on whether a single algorithm can achieve optimal shifting regret
for both adversarial and stochastic losses.
Note that although recent work by \citet{SaniNeLa14} also solves this open problem
in some sense,
their method requires knowing the number of partitions and other information ahead of time
and also gives a worse bound for stochastic losses,
while {\ANHTV} is completely parameter-free and gives optimal bounds.

We finally consider the most general case
where the competitor varies over time with no constraints, 
which subsumes the previous two examples (adaptive regret and shifting regret).
This problem was introduced in \citet{HerbsterWa01} and later generalized by
\citet{CesabianchiGaLuSt12}.
Their algorithm (fixed share) also requires knowing some information on 
the sequence of competitors to optimally tune parameters.
We avoid this issue by showing that while this problem seems more general and difficult,
it is in fact {\it equivalent to its special case}: achieving adaptive regret.
This equivalence theorem is independent of the concrete algorithms and may be of independent interest.
Applying this result, we show that without any parameter tuning,
{\ANHTV} automatically achieves a bound comparable to 
the one achieved by the optimally tuned fixed share algorithm when competing with time-varying competitors.

Concrete results and detailed comparisons on this first example 
can be found in Section \ref{sec:time_varying}.
To sum up, {\ANHTV} is an algorithm that is simultaneously adaptive in 
the number of experts, the competitors and the way the losses are generated.

The second example we provide is predicting almost as well as the best pruning tree
\citep{HelmboldSc97},
which was also shown to be reducible to a sleeping expert problem \citep{FreundScSiWa97}.
%In this case, each edge of the decision tree is a sleeping expert and 
%the competitor is a uniform distribution over the terminal edges of the best pruning tree,
%which is clearly unknown in advanced. 
Previous work either only considered the log loss setting, 
or assumed prior information on the best pruning tree is known.
Using {\ANH}, we again provide better or comparable bounds without knowing any prior information.
In fact, due to the adaptivity of {\ANH} in the number of experts,
our regret bound depends on the total number of distinct traversed edges so far,
instead of the total number of edges of the decision tree as in \citet{FreundScSiWa97}
which could be exponentially larger.
Concrete comparisons can be found in Section \ref{sec:pruning_tree}.

\paragraph{Related work.}
While competing with any unknown competitor simultaneously is relatively easy in the log loss setting 
\citep{LittlestoneWa94, AdamskiyKoChVo12, KoolenAdWa12},
it is much harder in the bounded loss setting studied here.
The well-known exponential weights algorithm gives the optimal results 
only when the learning rate is optimally tuned in terms of the competitor \citep{FreundSc99}.
\citet{ChernovVo10} also studied $\epsilon$-quantile regret, %for any $\epsilon$,
but no concrete algorithm was provided.
Several work considers competing with unknown competitors
in a different unconstrained linear optimization setting 
\citep{StreeterMc12, Orabona13, McmahanOr14, Orabona14}.
\citet{JadbabaieRaShSr15} studied general adaptive online learning algorithms against time-varying competitors,
but with different and incomparable measurement of the hardness of the problem.
%Expert algorithms that are optimal in both adversarial and stochastic setting
%were proposed in several recent work 
%\citep{DeRooijErGrKo14, VanErvenKoWa14, GaillardStEr14}.
As far as we know, none of the existing algorithms enjoys all the nice properties
discussed in this work at the same time as our algorithms do.

\section{The Expert Problem and NormalHedge.DT}\label{sec:setup}
In the expert problem, 
%a player tries to lose as little as possible
%by cleverly choosing an allocation over a set of options on each day.   
%Formally, the game proceeds for $T$ rounds, and 
on each round $t = 1, \ldots, T$: 
the player first chooses a distribution $\p_t$ over $N$ experts,
then the adversary decides each expert's loss $\ell_{t,i}\in [0,1]$, and reveals these losses to the player.
At the end of this round,  the player suffers the weighted average loss 
$\hl_t = \p_t\cdot\l_t$ with $\l_t = (\ell_{t,1}, \ldots, \ell_{t,N})$.
We denote the {\it instantaneous regret} to expert $i$ on round $t$ by $r_{t,i} = \hl_t - \ell_{t,i}$,
the cumulative regret by $R_{t,i} = \sum_{\tau=1}^t r_{\tau, i}$,
and the cumulative loss by $L_{t,i} = \sum_{\tau=1}^t \ell_{\tau, i} $.
Throughout the paper, a bold letter denotes a vector with $N$ corresponding coordinates.
For example, $\r_t$, $\R_t$ and $\L_t$ represent $(r_{t,1}, \ldots, r_{t,N})$,
$(R_{t,1}, \ldots, R_{t,N})$ and $(L_{t,1}, \ldots, L_{t,N})$ respectively.

Usually, the goal of the player is to minimize the regret to the best expert, that is,
$\max_{i} R_{T,i}$.
Here we consider a more general case where the player wants to minimize 
the regret to an arbitrary convex combination of experts: $R_T(\u) = \sum_{t=1}^T \u \cdot \r_t$
where the competitor $\u$ is a fixed unknown distribution over the experts.
In other words, this regret measures the difference between the player's loss
and the loss that he would have suffered if he used a constant strategy $\u$ all the time.
Clearly, $R_T(\u)$ can be written as $\u \cdot \R_T$ and can then be upper bounded
appropriately by a bound on each $R_{T,i}$ (for example, $\max_i R_{T,i}$).
However, our goal is to get a better and more refined bound on $R_T(\u)$ that depends on $\u$.
More importantly, we aim to achieve this without knowing the competitor $\u$ ahead of time.
When it is clear from the context, we drop the subscript $T$ in $R_T(\u)$.

In fact, in Section \ref{sec:time_varying}, 
we will consider an even more general notion of regret introduced in \citet{HerbsterWa01},
where we allow the competitor to vary over time and to have different scales.
Specifically, let $\u_1, \ldots, \u_T$ be $T$ different vectors with $N$ nonnegative coordinates
(denoted by $\u_{1:T}$). 
Then the regret of the player to this sequence of competitors is
$R(\u_{1:T}) = \sum_{t=1}^T \u_t \cdot \r_t$.
If all these competitors are distributions (which they are not required to be),
then this regret captures a very natural and general concept of comparing 
the player's strategy to any other strategy.
%assuming the adversary is oblivious (that is, losses are determined before the game starts).
Again, we are interested in developing low-regret algorithms that do not need to know 
any information of this sequence of competitors beforehand.

We briefly describe a recent algorithm for the expert problem,
NormalHedge.DT \citep{LuoSc14b} (a variant of NormalHedge \citep{ChaudhuriFrHs09}),
before we introduce our new improved variants.
On round $t$, NormalHedge.DT sets 
$ p_{t,i} \propto \exp\(\frac{[R_{t-1,i}+1]_+^2}{3t}\) - \exp\(\frac{[R_{t-1,i}-1]_+^2}{3t}\),  $
where $[x]_+ = \max\{0, x\}$.
Let $\epsilon \in (0, 1]$ and competitor $\u^*_\epsilon$ be a distribution 
that puts all the mass on the $\lceil N\epsilon \rceil$-th best expert,
that is, the one that ranks $\lceil N\epsilon \rceil$ among all experts according to 
their total loss $L_{T,i}$ from the smallest to the largest. 
Then the regret guarantee for NormalHedge.DT states
$R(\u^*_\epsilon) \leq O\Big(\sqrt{T\ln\(\tfrac{\ln T}{\epsilon}\)}\Big) $
simultaneously for all $\epsilon$,
which means the algorithm suffers at most this amount of regret 
for all but an $\epsilon$ fraction of the experts.
Note that this bound does not depend on $N$ at all.
This is the first concrete algorithm with this kind of adaptive property
(the original NormalHedge \citep{ChaudhuriFrHs09} still has a weak dependence on $N$).
In fact, as we will show later, one can even extend the results to any competitor $\u$.
Moreover, we will improve NormalHedge.DT so that 
it has a much smaller regret when the problem is ``easy'' in some sense.

\paragraph{Notation.}
We use $[N]$ to denote the set $\{1, \ldots, N\}$,
$\Delta_N$ to denote the simplex of all distributions over $[N]$,
and $\RE(\cdot \;||\; \cdot)$ to denote the relative entropy between two distributions,
Also define $\tL_{t,i} = \sum_{\tau=1}^t [\ell_{\tau,i} - \hl_\tau]_+$.
Many bounds in this work will be in terms of  $\tL_{T,i}$,
which is always at most $L_{T,i}$ since trivially $[\ell_{t,i} - \hl_t]_+ \leq \ell_{t,i}$.
We consider ``log log'' terms to be nearly constant, and use $\O()$ notation to hide these terms.
Indeed, as pointed out by \citet{ChernovVo10}, $\ln\ln x$ is smaller than $4$ even when 
$x$ is as large as the age of the universe expressed in microseconds ($\approx 4.3 \times 10^{17}$). 

\section{A New Algorithm: {\ANH}}\label{sec:ANH}
We start by writing NormalHedge.DT in a general form.
We define {\it potential function} $\Phi(R, C) = \exp\(\frac{[R]_+^2}{3C}\)$
with $\Phi(0,0)$ defined to be $1$,
and also a weight function with respect to this potential:
\[ w(R, C) = \frac{1}{2}\(\Phi(R+1, C+1) - \Phi(R-1, C+1)\). \]
Then the prediction of NormalHedge.DT is simply to set $p_{t,i}$ to be proportional to
$w(R_{t-1,i}, C_{t-1})$ where $C_t = t$ for all $t$.
Note that $C_t$ is closely related to the regret.
In fact, the regret is roughly of order $\sqrt{C_T}$ (ignoring the log term).
Therefore, in order to get an expert-wise and more refined bound,
we replace $C_t$ by $C_{t,i}$ for each expert so that 
it captures some useful information for each expert $i$. 
There are several possible choices for $C_{t,i}$
(discussed at the end of Appendix~\ref{app:proof_ANH}), 
but for now we focus on the one used in our new algorithm:
$C_{t,i}  = \sum_{\tau=1}^t |r_{\tau, i}|$, that is,
the cumulative magnitude of the instantaneous regrets up to time $t$.
We call this algorithm {\ANH} and summarize it in Algorithm \ref{alg:ANH}.
Note that we even allow the player to have a prior distribution $\q$ over the experts,
which will be useful in some applications as we will see in Section \ref{sec:time_varying}.
The theoretical guarantee of {\ANH} is stated below.

\begin{algorithm}[t]
\caption{\ANH}
\label{alg:ANH}
\begin{algorithmic}%[1]
\STATE {\bfseries Input:} A prior distribution $\q \in \Delta_N$ over experts (uniform if no prior available).
\STATE {\bfseries Initialize:} 
$\forall i\in[N], R_{0,i} = 0, C_{0,i} = 0$.
\FOR{$t=1$ {\bfseries to} $T$}
    \STATE Predict\footnote{If $\p_{t,i} \propto 0$ happens for all $i$, predict arbitrarily.} $p_{t,i} \propto q_i w(R_{t-1,i}, C_{t-1,i}) $.  
    \STATE Adversary reveals loss vector $\l_t$ and player suffers loss $\hl_t = \p_t \cdot \l_t$.
    \STATE Set $\forall i\in[N], r_{t,i} = \hl_t - \ell_{t,i}, R_{t,i} = R_{t-1,i} + r_{t,i}, C_{t,i} = C_{t-1,i} + |r_{t,i}|$.
\ENDFOR
\end{algorithmic}
\end{algorithm}

\begin{theorem}\label{thm:ANH}
The regret of {\ANH} to any competitor $\u \in \Delta_N$ is bounded as follows:
\begin{equation}\label{equ:regret}
R(\u) \leq  \sqrt{3(\u \cdot \C_T) \(\RE(\u \;||\; \q) + \ln B + \ln(1 + \ln N)\)}
= \O(\sqrt{(\u \cdot \C_T)\RE(\u \;||\; \q) }), 
\end{equation}
where $\C_T = (C_{T,1}, \ldots, C_{T,N})$, $B = 
1 + \frac{3}{2}\sum_i q_i \(1 + \ln (1 + C_{T,i})\) \leq
\frac{5}{2} + \frac{3}{2}\ln(1+T)$.
Moreover, %if there exists a subset $S \subseteq [N]$ of experts such that
%$u_i = 1/|S|$ if $i \in S$ and $0$ otherwise, 
if $\u$ is a uniform distribution over a subset of $[N]$,
%if $u_i = \frac{1}{|S|}$ when $i \in S$ and $0$ otherwise for some subset $S$ of $[N]$,
then the regret can be improved to
\begin{equation}\label{equ:improved_regret}
R(\u) \leq  \sqrt{3(\u \cdot \C_T) \(\RE(\u \;||\; \q) + \ln B + 1\)}. 
\end{equation}
\end{theorem}

Before we prove this theorem (see sketch at the end of this section and complete proof in Appendix \ref{app:proof_ANH}),
we discuss some implications of the regret bounds and why they are interesting.
First of all, the relative entropy term $\RE(\u \;||\; \q)$ captures how close
the player's prior is to the competitor. 
A bound in terms of $\RE(\u \;||\; \q)$ can be obtained, for example, using the classic 
exponential weights algorithm but requires carefully tuning the learning rate as a function of $\u$. 
Without knowing $\u$, as far as we know, 
{\ANH} is the only algorithm that can achieve this.\footnote{
In fact, one can also derive similar bounds for NormalHedge and NormalHedge.DT using our analysis.
See discussion at the end of Appendix~\ref{app:proof_ANH}.}

On the other hand, if $\q$ is a uniform distribution,
then using bound \eqref{equ:improved_regret} and the fact $C_{T,i} \leq T$, 
we get an $\epsilon$-quantile regret bound similar to the one of NormalHedge.DT:
%$$ \max_{i\in[N]} R_{T,i} \leq \sqrt{3T \(\ln N + \ln B + 1\)}, $$
%\begin{equation}\label{equ:zero_order_regret}
$R(\u^*_\epsilon) \leq  R(\u_{S_\epsilon}) \leq \sqrt{3T \(\ln \(\tfrac{1}{\epsilon}\) + \ln B + 1\)}$
%\end{equation}
where $\u_{S_\epsilon}$ is uniform over the top $\lceil N\epsilon \rceil$ experts. 
in terms of their total loss $L_{T,i}$.

However, the power of a bound in terms of $\C_T$ is far more than this.
\citet{GaillardStEr14} introduced a new second order bound that implies much smaller regret
when the problem is easy.
It turns out that our seemingly weaker first order bound is also enough 
to get the exact same results!
We state these implications in the following theorem which is essentially 
a restatement of Theorems~9~and~11 of \citet{GaillardStEr14} with weaker conditions.

\begin{theorem}\label{thm:implications}
Suppose an expert algorithm guarantees $R(\u) \leq \sqrt{(\u \cdot \C_T)A(\u)}$ 
where $A(\u)$ is some function of $\u$.
Then it also satisfies the following:
\begin{enumerate}
\item Recall $L_{T,i} = \sum_{t=1}^T \ell_{t,i}$ and  
$\tL_{T,i} = \sum_{t=1}^T [\ell_{t,i} - \hl_t]_+$. We have
$$ R(\u) \leq \sqrt{2(\u \cdot \TL_T)A(\u)} + A(\u) \leq 
\sqrt{2(\u \cdot \L_T)A(\u)} + A(\u)  .$$

\item Suppose the loss vector $\l_t$'s are independent random variables and
there exists an $i^*$ and some $\alpha \in (0,1]$ such that $\E[\ell_{t, i} - \ell_{t, i^*}] \geq \alpha$
for any $t$ and $i \neq i^*$. 
Let $\e_{i^*}$ be a distribution that puts all the mass on expert $i^*$.
Then we have $ \E[R_{T, i^*}] \leq  \frac{A(\e_{i^*})}{\alpha}, $
and with probability at least $1-\delta$,
$ R_{T, i^*} \leq \O\(\frac{A(\e_{i^*})}{\alpha} + 
\frac{1}{\alpha}\sqrt{A(\e_{i^*}) \ln\frac{1}{\delta}} \). $
\end{enumerate}
\end{theorem}

The proof of Theorem \ref{thm:implications} is based on the same idea as in \citet{GaillardStEr14},
and is included in Appendix \ref{app:implications} for completeness.
For {\ANH}, the term $A(\u)$ is $3(\RE(\u \;||\; \q) + \ln B + \ln(1 + \ln N))$ in general 
(or smaller for special $\u$ as stated in Theorem \ref{thm:ANH}).
Applying Theorem \ref{thm:implications} we have 
%\begin{equation}\label{equ:first_order_regret}
$R(\u) = \O\Big(\sqrt{(\u\cdot \tL_T)\RE(\u \;||\; \q)}\Big).\footnote{
We see $O(\RE(\u \;||\; \q))$ as a minor term %compared to $\O(\sqrt{(\u\cdot \tL_T)\RE(\u \;||\; \q)})$ 
and hide it in the big O notation,
which is not completely rigorous but will ease the presentation.
Same thing happens for other first order bounds in this work.
}$
 % + \RE(\u \;||\; \q)\Big).
%\end{equation}
Specifically, if $\q$ is uniform and assuming without loss of generality that $L_{T, 1} \leq \cdots \leq L_{T,N}$, 
then by a similar argument, % of Eq. \eqref{equ:zero_order_regret},
we have for {\ANH},
$ R(\u^*_\epsilon) \leq 
\O\Big(\sqrt{\tL_{T, \lceil N\epsilon\rceil} \ln\(\tfrac{1}{\epsilon}\)} \Big)$
for any $\epsilon$.
This answers the open question (in the affirmative) asked by 
\citet{ChaudhuriFrHs09} and \citet{ChernovVo10} on 
whether an improvement for small loss can be obtained for $\epsilon$-quantile regret
without knowing $\epsilon$.

On the other hand, when we are in a stochastic setting as stated in Theorem \ref{thm:implications},
{\ANH} ensures 
%\begin{equation}\label{equ:iid_regret}
$R_{T, i^*} \leq \O\(\tfrac{1}{\alpha}\ln \(\tfrac{1}{q_{i^*}}\)\)$
%\end{equation}
in expectation (or with high probability with an extra confidence term),
which does not grow with $T$.
Therefore, the new regret bound in terms of $\C_T$ actually leads to 
significant improvements compared to NormalHedge.DT.

\paragraph{Comparison to Adapt-ML-Prod \citep{GaillardStEr14}.}
Adapt-ML-Prod enjoys a second order bound in terms of $\sum_{t=1}^T r_{t,i}^2$,
which is always at most the term $\sum_{t=1}^T |r_{t,i}|$ appeared in our bounds.\footnote{We
briefly discuss the difficulty of getting a similar second order bound for our algorithm
at the end of Appendix~\ref{app:proof_ANH}.}
However, on one hand, as discussed above, 
these two bounds have the same improvements when the problem is easy in several senses;
on the other hand, 
Adapt-ML-Prod does not provide a bound in terms of $\RE(\u \;||\; \q)$ for an unknown $\u$.
In fact, as discussed at the end of Section A.3 of \citet{GaillardStEr14},
Adapt-ML-Prod cannot improve by exploiting a good prior $\q$ 
(or at least its current analysis cannot).
Specifically, while the regret for {\ANH} does not have an explicit dependence on $N$
and is much smaller when the prior $\q$ is close to the competitor $\u$, 
%(in terms of the relative entropy), 
the regret for Adapt-ML-Prod always has a $\ln N$ multiplicative term for $\sum_{t=1}^T r_{t,i}^2$,
which means even a good prior results in the same regret as a uniform prior!
More advantages of {\ANH} over Adapt-ML-Prod will be discussed
in concrete examples in following sections.

%\subsection{Proof of Theorem \ref{thm:ANH}}\label{subsec:proof_ANH}
\paragraph{Proof sketch of Theorem \ref{thm:ANH}.}
The analysis of NormaHedge.DT is based on the idea of converting the expert problem
into a drifting game \citep{Schapire01, LuoSc14b}.
Here, we extract and simplify the key idea of their proof and also improve it to form our analysis.
The main idea is to show that the weighted sum of potentials does not increase much on each round
using an improved version of Lemma 2 of \citet{LuoSc14b}.
In fact, we show that the final potential $\sum_{i=1}^N q_{i} \Phi(R_{T,i}, C_{T,i})$ is exactly bounded by $B$
(defined in Theorem~\ref{thm:ANH}).
From this, assuming without loss of generality that $q_1 \Phi(R_{T,1}, C_{T,1}) \geq \cdots \geq q_N \Phi(R_{T,N}, C_{T,N})$,
we have $ q_i \Phi(R_{T,i}, C_{T,i}) \leq \frac{B}{i}$ for all $i$,
which, by solving for $R_{T,i}$, gives $ R_{T,i} \leq \sqrt{3C_{T,i} \ln\(\tfrac{B}{i q_i}\)} $. 
Multiplying both sides by $u_i$, summing over $N$ and applying the Cauchy-Schwarz inequality,
we arrive at $ R(\u) \leq \sqrt{3(\u \cdot \C_T) (D(\u \;||\; \q) + \ln B)}, $
where we define $D(\u \;||\; \q) = \sum_{i=1}^N u_i\ln\(\tfrac{1}{i q_i}\)$.
It remains to show that $D(\u \;||\; \q)$ and $\RE(\u \;||\; \q)$ are close by standard analysis and Stirling's formula.

\section{Confidence-rated Advice and Sleeping Experts}\label{sec:sleeping}
In this section, we generalize {\ANH} to deal with experts that make confidence-rated advice,
a setting that subsumes many interesting applications as studied by \citet{Blum97} and \citet{FreundScSiWa97}.
In this general setting, on each round $t$,
each expert first reports its {\it confidence} $\I_{t,i} \in [0,1]$ for the current task.
The player then predicts $\p_t$ as usual with an extra yet natural restriction that
if $\I_{t,i} = 0$ then $p_{t,i} = 0$. 
That is, the player has to ignore those experts %(by putting zero weight on them) 
who abstain from making advice (by reporting zero confidence).
After that, the loss $\ell_{t,i}$ for those experts who did not abstain 
(i.e. $\I_{t,i} \neq 0$) are revealed and the player still suffers loss $\hl_t = \p_t \cdot \l_t$.
We redefine the instantaneous regret $r_{t,i}$ to be $\I_{t,i}(\hl_t - \ell_{t,i})$,
that is, the difference between the loss of the player and expert $i$ weighted by the confidence.
The goal of the player is, as before, to minimize cumulative regret to any competitor $\u$:
$R(\u) = \sum_{t \leq T} \u \cdot \r_t$.
Clearly, the classic expert problem that we have studied in previous sections
is just a special case of this general setting with $\I_{t,i} = 1$ for all $t$ and $i$.

Moreover, with this general form of $r_{t,i}$, 
{\ANH} can be used to deal with this general setting with only one simple change
of scaling the weights by the confidence:
\begin{equation}\label{equ:general_ANH}
p_{t,i} \propto q_i \I_{t,i} w(R_{t-1,i}, C_{t-1,i}),
\end{equation}
where $R_{t,i}$ and $C_{t,i}$ is still defined to be $\sum_{\tau=1}^t r_{\tau, i}$
and $\sum_{\tau=1}^t |r_{\tau, i}|$ respectively.
The constraint $\I_{t,i} = 0 \Rightarrow p_{t,i} = 0$ is clearly satisfied.
In fact, Algorithm $\ref{alg:ANH}$ can be seen as a special case of this general form 
of {\ANH} with $I_{t,i} \equiv 1$.
Furthermore, the regret bounds in Theorem \ref{thm:ANH} still hold without any changes,
which are summarized below (proof deferred to Appendix~\ref{app:proof_ANH}).

\begin{theorem}\label{thm:confidence_regret}
For the confidence-rated expert problem, regret bounds \eqref{equ:regret}
and \eqref{equ:improved_regret} still hold for general {\ANH} (Eq. \eqref{equ:general_ANH}).
\end{theorem} 

Previously, \citet{GaillardStEr14} studied a general reduction from an expert algorithm to a confidence-rated expert algorithm.
Applying those results here gives the exact same algorithm and regret guarantee mentioned above.
However, we point out that the general reduction is not always applicable.
Specifically, it is invalid if there is an unknown number of experts in the confidence-rated setting (explained more in the next paragraph)
while the expert algorithm in the standard setting requires knowing the number of experts as a parameter.
This is indeed the case for most algorithms
(including Adapt-ML-Prod and even the original NormalHedge by \citet{ChaudhuriFrHs09}).
{\ANH} naturally avoids this problem since it does not depend on $N$ at all.

\paragraph{Sleeping Experts.}
We are especially interested in the case when $\I_{t,i} \in \{0,1\}$,
also called the specialist/sleeping expert problem where $\I_{t,i} = 0$ means that
expert $i$ is ``asleep'' for round $t$ and not making any advice.
This is a natural setting where the total number of experts is unknown ahead of time.
Indeed, the number of awake experts can be dynamically changing over time.
An expert that has never appeared before should be thought of as being asleep for 
all previous rounds.

{\ANH} is a very suitable algorithm to deal with this case due to its independence of 
the total number of experts.
If an expert $i$ appears for the first time on round $t$,
then by definition it will naturally start with $R_{t-1,i} = 0$ and $C_{t-1,i} = 0$.
Although we state the prior $q$ as a distribution, 
which seems to require knowing the total number of experts,
it is not an issue algorithmically since $\q$ is only used to 
scale the unnormalized weights (Eq. \eqref{equ:general_ANH}).
For example, if we want $\q$ to be a uniform distribution over $N$ experts
where $N$ is unknown beforehand,
then to run {\ANH} we can simply treat $q_i$ in Eq. \eqref{equ:general_ANH}
to be $1$ for all $i$,
which clearly will not change the behavior of the algorithm anyway.
In this case, if we let $N_t$ denote the total number of distinct experts that 
have been seen up to time $t$
and the competitor $\u$ concentrates on any of these experts,
then the relative entropy term in the regret (up to time $t$)
will be $\ln N_t$ (instead of $\ln N$), which is changing over time.

Using the adaptivity of {\ANH} in both the number of experts and the competitor,
we provide improved results for two instances of the sleeping expert problem in the next two sections.

\section{Time-Varying Competitors}\label{sec:time_varying}
In this section, 
we study a more challenging goal of competing with time-varying competitors in
the standard expert setting (that is, each expert is always awake and again $r_{t,i} = \hl_t - \ell_{t,i}$),
which turns out to be reducible to a sleeping expert problem.
Results for this section are summarized in Table~\ref{tab:time_varying}.

\subsection{Special Cases: Adaptive Regret and Tracking the Best Expert}
We start from a special case: adaptive regret, 
introduced by \citet{HazanSe07} to better capture changing environments.
Formally, consider any time interval $t = t_1, \ldots, t_2$,
and let $R_{[t_1,t_2], i} = \sum_{t = t_1}^{t_2} r_{t,i}$ be the regret to expert $i$ on this interval
(similarly define $L_{[t_1, t_2], i} =  \sum_{t = t_1}^{t_2} \ell_{t,i}$
and  $C_{[t_1, t_2], i} =  \sum_{t = t_1}^{t_2} |r_{t,i}|$).
The goal of the player is to obtain relatively small regret on any interval.
\citet{FreundScSiWa97} essentially introduced a way to reduce this problem 
to a sleeping expert problem,
which was later improved by \citet{AdamskiyKoChVo12}.
Specifically, for every pair of time $t$ and expert $i$, 
we create a sleeping expert, denoted by $(t,i)$,
who is only awake after (and including) round $t$
and since then suffers the same loss as the original expert $i$.
So we have $Nt$ sleeping experts in total on round $t$. 
The prediction $p_{t, i}$ is set to be the sum of all the weights of 
sleeping expert $(\tau, i) \; (\tau = 1, \ldots, t)$.
It is clear that doing this ensures that the cumulative regret up to time $t_2$ 
with respect to sleeping expert $(t_1, i)$ is exactly $R_{[t_1,t_2], i}$ in the original problem.

This is a sleeping expert problem for which {\ANH} is very suitable,
since the number of sleeping experts keeps increasing and
the total number of experts is in fact unknown if the horizon $T$ is unknown.
Theorem \ref{thm:confidence_regret} implies that
the resulting algorithm gives the following adaptive regret:
$$ R_{[t_1,t_2], i} = \O\(\sqrt{\(\textstyle\sum_{t = t_1}^{t_2} |r_{t,i}|\) \ln\(\tfrac{1}{q_{(t_1,i)}}\)}\)
= \O\(\sqrt{\(\textstyle\sum_{t = t_1}^{t_2} |r_{t,i}|\) \ln\(N t_1\)}\),$$
where $\q$ is a prior over the $N t_2$ experts
and the last step is by setting the prior to be $q_{(t,i)} \propto 1/t^2$ for all $t$ and $i$.\footnote{
Note that as discussed before, the fact that $t_2$ is unknown and thus 
$\q$ is unknown does not affect the algorithm.}
This prior is better than a simple uniform distribution
which leads to a term $\ln(N t_2)$ instead of $\ln(N t_1)$.
We call this algorithm {\ANHTV}.\footnote{``TV'' stands for ``time-varying''.}
To be concrete, on round $t$ {\ANHTV} predicts
$ p_{t,i} \propto \sum_{\tau=1}^t \frac{1}{\tau^2} 
w\(R_{[\tau, t-1],i}, C_{[\tau, t-1],i}\).$

Again, Theorem \ref{thm:implications} can be applied to get a more
interpretable bound $\O\Big(\sqrt{\tL_{[t_1, t_2], i} \ln\(N t_1\)}\Big)$
where $\tL_{[t_1, t_2], i} = \sum_{t = t_1}^{t_2} [\ell_{t,i} - \hl_{t,i}]_+ \leq L_{[t_1, t_2], i}$,
and a much smaller bound $\O\(\frac{\ln\(N t_1\)}{\alpha}\)$ if the losses are stochastic 
on interval $[t_1, t_2]$ in the sense stated in Theorem  \ref{thm:implications}.

One drawback of {\ANHTV} is that its time complexity per round is $O(Nt)$
and the overall space is $O(NT)$.
However, the data streaming technique used in \citet{HazanSe07} can be 
directly applied here to reduce the time and space complexity to $O(N\ln t)$ and $O(N\ln T)$ respectively,
with only an extra multiplicative $O(\sqrt{\ln(t_2-t_1)})$ factor in the regret.
%We include this fast version of {\ANHTV} and its analysis (Theorem \ref{thm:fast_ANHTV}) 
%in Appendix \ref{app:ANHTV} for completeness.

\paragraph{Tracking the best expert.}
In fact, {\ANHTV} is a solution for one of the open problems proposed by 
\citet{WarmuthKo14}.
Adaptive regret immediately implies the so-called $K$-shifting regret 
for the problem of tracking the best expert in a changing environment.
Formally, define the $K$-shifting regret $R_{K\text{-Shift}}$
to be $\max \sum_{k=1}^K R_{[t_{k-1}+1, t_k], i_k}$ where 
the max is taken over all $i_1, \cdots, i_K \in [N]$ and $0 = t_0 < t_1 < \cdots < t_K = T$.
In other words, the player is competing with the best $K$-partition of the whole game
and the best expert on each of these partitions.
Let $L^*_{K\text{-Shift}} = \max \sum_{k=1}^K  L_{[ t_{k-1}+1, t_k], i_k}$
be the total loss of such best partition (that is, the max is taken over the same space),
and similarly define $\tL^*_{K\text{-Shift}} = \max\sum_{k=1}^K  \tL_{[ t_{k-1}+1, t_k], i_k} \leq 
L^*_{K\text{-Shift}}$. % where $t_k, i_k \; (k \in [K])$ is the best partition. 
Since essentially $R_{K\text{-Shift}}$ is just the sum of $K$ adaptive regrets,
using the bounds discussed above and the Cauchy-Schwarz inequality, 
we conclude that  {\ANHTV} ensures 
$R_{K\text{-Shift}} = \O\Big(\sqrt{K \tL^*_{K\text{-Shift}} \ln(NT)}\Big). $
Also, if the loss vectors are generated randomly on these $K$ intervals,
each satisfying the condition stated in Theorem \ref{thm:implications},
%(that is, there is a gap $\alpha$ between one specific expert and the others in expectation),
then the regret is
$R_{K\text{-Shift}} = \O\(\frac{K\ln(NT)}{\alpha} \) $
in expectation (high probability bound is similar).
These bounds are optimal up to logarithmic factors \citep{HazanSe07}.
This is exactly what was asked in \citet{WarmuthKo14}: 
whether there is an algorithm that
can do optimally for both adversarial and stochastic losses in
the problem of tracking the best expert.
{\ANHTV} achieves this goal without knowing $K$ or any other information,
while the solution provided by \citet{SaniNeLa14} needs to know
$K$ , $L^*_{K\text{-Shift}}$ and $\alpha$ to get the same adversarial bound
and a worse stochastic bound of order $O(1/\alpha^2)$.

\paragraph{Comparison to previous work.}
For adaptive regret, % $R_{[t_1, t_2], i}$, 
the FLH algorithm by \citet{HazanSe07} treats any standard expert algorithm as a sleeping expert, 
and has an additive term $O(\sqrt{t_2}\ln t_2)$ in addition to the base algorithm's regret
(when no prior information is available),
which adds up to a large $O(K\sqrt{T}\ln T)$ term for $K$-shifting regret.
Due to this extra additive regret, 
FLH also does not enjoy first order bounds nor small regret in the stochastic setting,
even if the base algorithm that it builds on provides these guarantees.
On the other hand, FLH was proposed to achieve adaptive regret
for any general online convex optimization problem.
We point out that using {\ANH} as the master algorithm in their framework
will give similar improvements as discussed here.
%Details are postponed to Appendix \ref{app:adaptive_OCO} due to lack of space.

Adapt-ML-Prod is not directly applicable here for the corresponding sleeping expert problem
since the total number of experts is unknown.

Another well-studied algorithm for this problem is ``fixed share''.
Several works on fixed share for the simpler ``log loss'' setting were studied
before \citep{HerbsterWa98, BousquetWa03, AdamskiyKoChVo12, KoolenAdWa12}.
\citet{CesabianchiGaLuSt12} studied a generalized fixed share algorithm 
for the bounded loss setting considered here.
When $t_2 - t_1$ and $L_{[t_1, t_2], i}$ are known, their algorithm ensures 
$R_{[t_1, t_2], i} = O\(\sqrt{L_{[t_1, t_2], i} \ln (N(t_2-t_1))}\)$ for adaptive regret,
and when $K$, $T$ and $L^*_{K\text{-Shift}}$ are known, 
they have $R_{K\text{-Shift}} = O\(\sqrt{K L^*_{K\text{-Shift}} \ln(NT/K)}\)$.
No better result is provided for the stochastic setting. 
More importantly, when no prior information is known,
which is the case in practice, the best results one can extract from their analysis are
$R_{[t_1, t_2], i} = O(\sqrt{t_2\ln (N t_2)})$ and $R_{K\text{-Shift}} = O(K\sqrt{T\ln (NT)})$,
which are much worse than our results.
%On the other hand,
%the time and space complexity for fixed share are better ($O(N)$ compared to $O(N\ln T)$ for our algorithm).

\begin{table}[t]
\caption{Comparison of different algorithms for time-varying competitors\footnote{
For fair comparison, we only consider the case when no prior information
(e.g. $t_1, t_2, V(\u_{1:T})$ etc) is known.}}
\label{tab:time_varying}
\vskip 0.15in
\begin{center}
\begin{small}
\def\arraystretch{1}
\begin{tabular}{|c|c|c|c|}
\hline
Algorithm & $R_{[t_1,t_2], i}$ & $R_{K\text{-Shift}}$ & $R(\u_{1:T})$   \\

\hline
\specialcell{\ANHTV \\ (this work)} & 
$\sqrt{\tL_{[t_1, t_2], i} \ln\(N t_1\)}$ & 
$\sqrt{K \tL^*_{K\text{-Shift}} \ln(NT)}$ &
$\sqrt{V(\u_{1:T})\tL(\u_{1:T})\ln \(NT\)}$ \\

\hline
\specialcell{FLH\footnote{The choice of the base algorithm %for FLH
does not matter since the dominating term in the regret comes from FLH itself.} \\ 
{\scriptsize\citep{HazanSe07}}} & 
$\sqrt{t_2}\ln t_2$ & 
$K\sqrt{T}\ln T$ & 
unknown \\

\hline
\specialcell{Fixed Share  \\ {\scriptsize\citep{CesabianchiGaLuSt12}}} & 
$\sqrt{t_2\ln (N t_2)}$ & 
$K\sqrt{T\ln (NT)}$ & 
unknown \\

\hline
\end{tabular}
\end{small}
\end{center}
\vskip -0.1in
\end{table}

\subsection{General Time-Varying Competitors}
We finally discuss the most general goal: 
compete with different $\u_t$ on different rounds. 
Recall $R(\u_{1:T}) = \sum_{t=1}^T \u_t \cdot \r_t$ where $\u_t \in \mathbb{R}_+^N$ for all $t$
(note that $\u_t$ does not even have to be a distribution).
Clearly, adaptive regret and $K$-shifting regret are 
special cases of this general notion.
Intuitively, how large this regret is should be closely related to 
how much the competitor's sequence $\u_{1:T}$ varies.
\citet{CesabianchiGaLuSt12} introduced a distance measurement to capture this variation:
$V(\u_{1:T}) = \sum_{t=1}^T \sum_{i = 1}^N [u_{t,i} - u_{t-1,i}]_+$
where we define $u_{0,i} = 0$ for all $i$.
Also let $\|\u_{1:T}\| = \sum_{t=1}^T \|\u_t\|_1$ and 
$L(\u_{1:T}) = \sum_{t=1}^T \u_t \cdot \l_t$.
Fixed share is shown to ensure the following regret \citep{CesabianchiGaLuSt12}:
$ R(\u_{1:T}) = O\big(\sqrt{V(\u_{1:T})L(\u_{1:T})\ln \(N\|\u_{1:T}\|/V(\u_{1:T})\)}\big) $
when $V(\u_{1:T}), L(\u_{1:T})$ and $\|\u_{1:T}\|$ are known.
No result was provided otherwise.\footnote{
Although in Section 7.3 of \citet{CesabianchiGaLuSt12}, 
the authors mentioned online tuning technique for the parameters, 
it only works for special cases (e.g. adaptive regret).}
Here, we show that our parameter-free algorithm {\ANHTV} actually 
achieves almost the same bound without knowing any information beforehand.
Moreover, while the results in \citet{CesabianchiGaLuSt12} 
are specific for the fixed share algorithm,
we prove the following results which are independent of the concrete algorithms
and may be of independent interest.

\begin{theorem}\label{thm:time_varying}
Suppose an expert algorithm ensures $R_{[t_1, t_2], i} \leq \sqrt{A\sum_{t = t_1}^{t_2} z_{t,i}}$ 
for any $t_1, t_2$ and $i$, 
where $z_{t,i} \geq 0$ can be anything depending on $t$ and $i$ 
(e.g. $|r_{t,i}|$, $[\ell_{t,i} - \hl_t]_+$, $\ell_{t,i}$ or constant $1$),
and $A$ is a term independent of $t_1, t_2$ and $i$.
Then this algorithm also ensures 
$$ R(\u_{1:T}) \leq \sqrt{A V(\u_{1:T}) \textstyle\sum_{t=1}^T \u_t\cdot\z_t}. $$
Specially, for {\ANHTV}, plugging $A = \O(\ln(NT))$ and $z_{t,i} = [\ell_{t,i} - \hl_t]_+$ gives
$ R(\u_{1:T}) = \O\Big(\sqrt{V(\u_{1:T})\tL(\u_{1:T})\ln \(NT\)}\Big),  $
where $\tL(\u_{1:T}) = \sum_{t=1}^T \sum_{i=1}^N u_{t,i}[\ell_{t,i} - \hl_t]_+$.
\end{theorem} 

The key idea of the proof is to rewrite $R(\u_{1:T})$ as a weighted sum of
several adaptive regrets in an optimal way 
(see Appendix \ref{app:time_varying} for the complete proof).
This theorem tells us that while playing with time-varying competitors seems
to be a harder problem,
it is in fact not any harder than its special case: achieving adaptive regret on any interval. 
Although the result is independent of the algorithms,
one still cannot derive bounds on $R(\u_{1:T})$ for FLH or fixed share 
based on their adaptive regret bounds,
because when no prior information is available, 
the bounds on $R_{[t_1, t_2], i}$ for these algorithms are of order $O(\sqrt{t_2})$
instead of $O(\sqrt{t_2-t_1})$, which is not good enough.
We refer the reader to Table \ref{tab:time_varying} for a summary of this section.

\section{Competing with the Best Pruning Tree}\label{sec:pruning_tree}
We now turn to our second application on predicting almost as well as 
the best pruning tree within a template tree.
This problem was studied in the context of online learning by \citet{HelmboldSc97}
using the approach of \citet{WillemsShTj93, WillemsShTj95}.
It is also called the tree expert problem in \citet{CesabianchiLu06}.
\citet{FreundScSiWa97} proposed a generic reduction from a tree expert problem
to a sleeping expert problem. 
Using this reduction with our new algorithm, 
we provide much better results compared to previous work
(summarized in Table \ref{tab:tree_expert}).

Specifically, consider a setting where on each round $t$,
the predictor has to make a decision $y_t$ from some convex set $\Y$
given some side information $x_t$,
and then a convex loss function $f_t : \Y \rightarrow [0,1]$ is revealed
and the player suffers loss $f_t(y_t)$.
The predictor is given a {\it template tree} $\G$ to consult.
Starting from the root, 
each node of $\G$ performs a test on $x_t$ to decide 
which of its child should perform the next test,
until a leaf is reached.
In addition to a test, each node (except the root) also makes a prediction based on $x_t$.
A {\it pruning tree} $\P$ is a tree induced by replacing 
zero or more nodes (and associated subtrees) of $\G$ by leaves. 
%(we also rule out the case of having just the root as a pruning tree).
The prediction of a pruning tree $\P$ given $x_t$, denoted by $\P(x_t)$, is naturally defined as
the prediction of the leaf that $x_t$ reaches by traversing $\P$. %starting from the root.
The player's goal is thus to predict almost as well as the best pruning tree in hindsight,
that is, to minimize $R_{\G} = \sum_{t=1}^T f_t(y_t) - \min_{\P}\sum_{t=1}^T f_t(\P(x_t))$.

%Naively this can be solved by treating each pruning tree as an expert
%and then run any expert algorithm.
%The problem is that the number of experts will be exponentially large, 
%and thus the running time and space are unacceptable.
The idea of the reduction introduced by \citet{FreundScSiWa97} 
is to view each edge of $\G$ as a sleeping expert (indexed by $e$),
who is awake only when traversed by $x_t$, and in that case
predicts $y_{t,e}$, the same prediction as the child node that it connects to.
The predictor runs a sleeping expert algorithm with loss $\ell_{t, e} = f_t(y_{t,e})$,
and eventually predicts $y_t = \sum_{e \in E} p_{t,e} y_{t,e}$
where $E$ denotes the set of edges of $\G$
and $p_{t,e} \;(e \in E)$ is the output of the expert algorithm;
thus by convexity of $f_t$, we have 
$f_t(y_t) \leq \sum_{e \in E} p_{t,e}f_t(y_{t,e}) = \hl_t$.
Note that we only care about the predictions of those awake experts
since otherwise $p_{t,e}$ is required to be zero.
Now let $\P^*$ be one of the best pruning trees, that is, 
$\P^* \in \arg\min_{\P}\sum_{t=1}^T f_t(\P(x_t))$,
and $m$ be the number of leaves of $\P^*$.
In the expert problem, we will set the competitor $\u$ to be a uniform distribution 
over the $m$ terminal edges (that is, the ones connecting the leaves) of $\P^*$,
and the prior $\q$ to be a uniform distribution over all the edges.
Since on each round, one and only one of those $m$ experts is awake,
and its prediction is exactly $\P^*(x_t)$,
we have $ R(\u) = \sum_{t=1}^T \frac{1}{m} (\hl_t - f_t(\P^*(x_t)))$,
and therefore $R_{\G} \leq m R(\u)$.

It remains to pick a concrete sleeping algorithm to apply.
There are two reasons that make {\ANH} very suitable for this problem.
First, since $m$ is clearly unknown ahead of time,
we are competing with an unknown competitor $\u$, 
which is exactly what {\ANH} can deal with.
Second, the number of awake experts is dynamically changing,
and as discussed before, in this case
{\ANH} enjoys a regret bound that is adaptive in the the number of experts seen so far.
Formally, recall the notation $N_t$,
which in this case represents the total number of distinct traversed edges up to round $t$.
Then by Theorem \ref{thm:confidence_regret}, 
we have 
$$ R_{\G} = \O(m \sqrt{(\u \cdot \C_T) \RE(\u \;||\; \q)}) 
=  \O\bigg(\sqrt{m \(\textstyle\sum_{t=1}^T |\hl_t - f_t(\P^*(x_t))|\) \ln \(\tfrac{N_T}{m}\) }\bigg), $$
which, by Theorem \ref{thm:implications}, implies 
$ R_{\G} = \O\Big(\sqrt{m \tL^*\ln (N_T/m)} \Big)$ where 
$\tL^* = \sum_{t=1}^T [f_t(\P^*(x_t)) - \hl_t]_+$,
which is at most the total loss of the best pruning tree $L^* = \sum_{t=1}^T f_t(\P^*(x_t))$.
Moreover, the algorithm is efficient: 
the overall space requirement is $O(N_T)$, 
and the running time on round $t$ is $O(\|x_t\|_{\G})$ 
where we use $\|x_t\|_{\G}$ to denote the number of edges that $x_t$ traverses.

\begin{table}[t]
\caption{Comparison of different algorithms for the tree expert problem}
\label{tab:tree_expert}
\vskip 0.05in
\begin{center}
\begin{small}
\begin{tabular}{|c|c|c|c|c|}
\hline
Algorithm & $R_{\G}$ & Time (per round) & Space (overall) &
\specialcell{Need to know \\$L^*$ and $m$?}  \\

\hline
\specialcell{\ANH \\ (this work)} & 
$\O\Big(\sqrt{m \tL^*\ln \(\tfrac{N_T}{m}\)}\Big)$ & 
$O\(\|x_t\|_{\G}\)$ & $O\(N_T\)$ & No \\

\hline
\specialcell{Adapt-ML-Prod \\ {\scriptsize\citep{GaillardStEr14}}} & 
$\O\big(\sqrt{m \tL^*\ln N}\big)$ & 
$O\(\|x_t\|_{\G}\)$ & $O\(N_T\)$ & No \\

\hline
\specialcell{Exponential Weights  \\ {\scriptsize\citep{HelmboldSc97}}} & 
$O\big(\sqrt{m L^*}\big)$ & 
$O\(d\|x_t\|_{\G}\)$ & $O\(N_T\)$ & Yes \\

\hline
\end{tabular}
\end{small}
\end{center}
\vskip -0.1in
\end{table}

\paragraph{Comparison to other solutions.}
The work by \citet{FreundScSiWa97} considers
a variant of the exponential weights algorithm in a ``log loss'' setting,
and is not directly applicable here
(specifically it is not clear how to tune the learning rate appropriately).
A better choice is the Adapt-ML-Prod algorithm by \citet{GaillardStEr14} 
(the version for the sleeping expert problem). 
%where different learning rates are used for different experts and rounds.
However, there is still one issue for this algorithm:
it does not give a bound in terms of $\RE(\u \;||\; \q)$ for an unknown $\u$.\footnote{
In fact, even if it does, this term is still dominated by a $\ln N$ term.
See the discussion at the end of Section A.3 of \citet{GaillardStEr14}
that we already mentioned at Section \ref{sec:ANH}.}
So to get a bound on $R(\u)$, the best thing to do is to use the definition $R(\u) = \u\cdot \R_T$
and a bound on each $R_{T,i}$.
In short, one can verify that Adapt-ML-Prod ensures regret 
$R(\u) = \O\(\sqrt{m \tL^*\ln N} \)$ where $N = |E|$ is the total number of edges/experts.
%The fact that the bound depends on $N$ instead of $N_T \leq N$
%is due to the direct reduction from an expert algorithm to 
%a sleeping expert algorithm used in their work,
%which was discussed previously in the paragraph after Theorem \ref{thm:confidence_regret}.
We emphasize that $N$ can be much larger than $N_T$ when the tree is huge.
Indeed, while $N_T$ is at most $T$ times the depth of $\G$, 
$N$ could be exponentially large in the depth.
The running time and space of Adapt-ML-Prod for this problem, however, is the same as {\ANH}.

We finally compare with a totally different approach \citep{HelmboldSc97},
where one simply treats each pruning tree as an expert and run the exponential weights algorithm.
Clearly the number of experts is exponentially large, 
and thus the running time and space are unacceptable by a naive implementation.
This issue is avoided by using a clever dynamic programming technique.
If $L^* $ and $m$ are known ahead of time, 
then the regret for this algorithm is $O(\sqrt{m L^*})$ by tuning the learning rate
optimally. % in terms of $L^*$ and $m$.
As discussed in \citet{FreundScSiWa97}, the linear dependence on $m$ in this bound
is much better than the one of the form  $O\(\sqrt{m L^*\ln N} \)$,
which, in the worst case, is linear in $m$.
This was considered as the main drawback of using the sleeping expert approach. 
However, the bound for {\ANH} is $\O\Big(\sqrt{m \tL^*\ln (N_T/m)} \Big)$,
which is much smaller as discussed previously and in fact comparable to $O(\sqrt{m L^*})$.
More importantly, $L^*$ and $m$ are unknown in practice.
In this case, no sublinear regret is known for this dynamic programming approach, since it
relies heavily on the fact that the algorithm is using a fixed learning rate
and thus the usual time-varying learning rate methods cannot be applied here. 
%and thus doing a simple multiplicative update per round.
%It is not clear how the usual time-varying learning rate methods,
%which are important for getting rid of the dependence on knowing $L^*$,
%can be applied here.
Therefore, although theoretically this approach gives small regret,
it is not a practical method.
The running time is also slightly worse than the sleeping expert approach.
For simplicity, suppose every internal node has $d$ children.
Then the time complexity per round is $O(d\|x\|_{\G})$.
The overall space requirement is %$O(N)$ for a naive implementation,
%but can be easily improved to 
$O(N_t)$, the same as other approaches.
Again, see Table \ref{tab:tree_expert} for a summary of this section.

Finally, as mentioned in \citet{FreundScSiWa97}, 
the sleeping expert approach can be easily generalized to predicting with a decision graph.
In that case, {\ANH} still enjoys all the improvements discussed in this section (details omitted).

% Acknowledgments---Will not appear in anonymized version
%\acks{}

\newpage
\bibliography{ref}
\bibliographystyle{plainnat}

\newpage
\appendix
\section{Complete proofs of Theorem~\ref{thm:ANH} and \ref{thm:confidence_regret}}\label{app:proof_ANH}

We need the following two lemmas. 
The first one is an improved version of Lemma 2 of \citet{LuoSc14b}.

\begin{lemma}\label{lem:potential}
For any $R \in \mathbb{R}, C \geq 0$ and $r \in [-1, 1]$,
we have $$\Phi(R + r, C + |r|)  \leq \Phi(R, C) + w(R, C)r + \frac{3|r|}{2(C+1)}.$$
\end{lemma}
\begin{proof}
We first argue that $\Phi(R + r, C + |r|)$, as a function of $r$,
is piecewise-convex on $[-1, 0]$ and $[0,1]$.
Since the value of the function is $1$ when $R+r < 0$ and is at least $1$ otherwise.
It suffices to only consider the case when $R+r \geq 0$.
On the interval $[0,1]$, we can rewrite the exponent (ignoring the constant $\frac{1}{3}$) as:
$$ \frac{(R+r)^2}{C+r} = (C+r) + \frac{(R-C)^2}{C+r} + 2(R-C), $$
which is convex in $r$.
Combining with the fact that ``if $g(x)$ is convex then $\exp(g(x))$ is also convex''
proves that $\Phi(R + r, C + |r|)$ is convex on $[0,1]$.
Similarly when $r \in [-1,0]$, rewriting the exponent as
$$ \frac{(R+r)^2}{C-r} = (C-r) + \frac{(R+C)^2}{C-r} - 2(R+C) $$
completes the argument.

Now define function $f(r) = \Phi(R + r, C + |r|) - w(R, C)r$.
Since $f(r)$ is clearly also piecewise-convex on $[-1, 0]$ and $[0,1]$,
we know that the curve of $f(r)$ is below the segment connecting points 
$(-1, f(-1))$ and $(0, f(0))$ on $[-1,0]$,
and also below the segment connecting points $(1, f(1))$ and $(0, f(0))$ on $[0,1]$.
This can be mathematically expressed as:
$$ f(r) \leq \max\{f(0) + (f(0) - f(-1))r, f(0) + (f(1) - f(0))r \} = f(0) + (f(1) - f(0))|r|, $$
where we use the fact $f(-1) = f(1)$. 
Now by Lemma 2 of \citet{LuoSc14b}, we have
$$ f(1) - f(0) = \tfrac{1}{2}\(\Phi(R+1, C+1) + \Phi(R-1, C+1)\) - \Phi(R, C) 
\leq \tfrac{1}{2}\(\exp\(\tfrac{4}{3(C+1)}\) - 1\), $$ 
which is at most $\frac{e^{\frac{4}{3}} - 1}{2(C+1)}$ since $C$ is nonnegative
and $e^x - 1 \leq \frac{e^a-1}{a}x$ for any $x\in[0,a]$.
Noting that $e^{\frac{4}{3}} - 1 \leq 3$ completes the proof.
\end{proof}

The second lemma makes use of Lemma \ref{lem:potential} to show that 
the weighted sum of potentials does not increase much and thus the final potential is relatively small.

\begin{lemma}\label{lem:final_potential}
{\ANH} ensures $ \sum_{i=1}^N q_{i} \Phi(R_{T,i}, C_{T,i}) \leq B
= 1 + \frac{3}{2}\sum_{i=1}^N q_i \(1 + \ln (1 + C_{T,i})\)$.
\end{lemma}
\begin{proof}
First note that since {\ANH} predicts $p_{t,i} \propto q_i w(R_{t-1,i}, C_{t-1,i})$, we have 
\begin{equation}\label{equ:DG_constraint}
\textstyle\sum_{i=1}^N q_i w(R_{t-1,i}, C_{t-1,i}) r_{t,i} = 0. 
\end{equation}
Now applying Lemma \ref{lem:potential} with $R = R_{t-1,i}, C = C_{t-1,i}$ and $r = r_{t,i}$,
multiplying the inequality by $q_i$ on both sides and summing over $i$ gives
$ \sum_{i=1}^N q_{i} \Phi(R_{t,i}, C_{t,i}) \leq \sum_{i=1}^N q_{i} \Phi(R_{t-1,i}, C_{t-1,i}) 
+ \frac{3}{2}\sum_{i=1}^N  \frac{q_{i}|r_{t,i}|}{C_{t-1,i}+1} .$
We then sum over $t \in [T]$ and telescope to show
$  \sum_{i=1}^N q_{i} \Phi(R_{T,i}, C_{T,i}) \leq 1 + 
\frac{3}{2}\sum_{i=1}^N q_{i} \sum_{t=1}^T  \frac{|r_{t,i}|}{C_{t-1,i}+1}. $
Finally applying Lemma 14 of \cite{GaillardStEr14} to show 
$ \sum_{t=1}^T  \frac{|r_{t,i}|}{C_{t-1,i}+1}$ $\leq 1 + \ln(1 + C_{T,i})$
completes the proof.
\end{proof}

We are now ready to prove Theorem \ref{thm:ANH} and Theorem~\ref{thm:confidence_regret}.
\vspace{5pt}
\begin{proof}(of Theorem \ref{thm:ANH})
$\quad$ Assume $q_1 \Phi(R_{T,1}, C_{T,1}) \geq \cdots \geq q_N \Phi(R_{T,N}, C_{T,N})$
without loss of generality.
Then by Lemma \ref{lem:final_potential}, it must be true that
$ q_i \Phi(R_{T,i}, C_{T,i}) \leq \frac{B}{i}$ for all $i$,
which, by solving for $R_{T,i}$, gives $ R_{T,i} \leq \sqrt{3C_{T,i} \ln\(\tfrac{B}{i q_i}\)} $. 
Multiplying both sides by $u_i$, summing over $N$ and applying the Cauchy-Schwarz inequality,
we arrive at
$ R(\u) \leq  \sum_{i=1}^N  \sqrt{3u_iC_{T,i} \cdot u_i\ln\(\tfrac{B}{i q_i}\)} 
\leq \sqrt{3(\u \cdot \C_T) (D(\u \;||\; \q) + \ln B)}, $
where we define $D(\u \;||\; \q) = \sum_{i=1}^N u_i\ln\(\tfrac{1}{i q_i}\)$.
It remains to show that $D(\u \;||\; \q)$ and $\RE(\u \;||\; \q)$ are close.
Indeed, we have $D(\u \;||\; \q) - \RE(\u \;||\; \q) = \sum_{i=1}^N u_i\ln\(\tfrac{1}{i u_i}\)$,
which, by standard analysis, can be shown to reach its maximum when $u_i \propto \frac{1}{i}$
and the maximum value is $\ln \sum_i \frac{1}{i} \leq \ln(1 + \ln N)$.
This completes the proof for Eq. \eqref{equ:regret}.

Finally, when $\u$ is in the special form as described in Theorem \ref{thm:ANH},
we have $D(\u \;||\; \q) - \RE(\u \;||\; \q) = \ln |S| + \frac{1}{|S|}\sum_{i\in S} \ln\(\tfrac{1}{i}\)
\leq \ln |S| - \frac{1}{|S|}\ln (|S|!)$.
By Stirling's formula $x! \geq \sqrt{2\pi x}\(\frac{x}{e}\)^x$, we arrive at
$D(\u \;||\; \q) - \RE(\u \;||\; \q) \leq 1 - (\ln\sqrt{2\pi|S|})/|S| \leq 1$,
proving Eq. \eqref{equ:improved_regret}.
\end{proof}

\begin{proof}(of Theorem \ref{thm:confidence_regret})
It suffices to point out that $r_{t,i}$ is still in the interval $[-1,1]$ and 
Eq. \eqref{equ:DG_constraint} in the proof of Lemma \ref{lem:final_potential} 
still holds by the new prediction rule Eq. \eqref{equ:general_ANH}.
The entire proof for Theorem \ref{thm:ANH} applies here exactly.
\end{proof}

The algorithm and the proof can be generalized to $C_{t,i} = \sum_{\tau=1}^t |r_{\tau,i}|^d$
for any $d \in [0,1]$.
Indeed, the only extra work is to prove the convexity of $\Phi(R+r, C+|r|^d)$.
When $d = 0$, we recover NormalHedge.DT exactly and get a bound on $R(\u)$ for any $\u$ 
(in terms of $\sqrt{T}$),  instead of just $R(\u^*_\epsilon)$ as in the original work.
It is clear that $d = 1$ gives the smallest bound, which is why we use it in {\ANH}.
The ideal choice, however, should be $d = 2$ so that a second order bound 
similar to the one of \citet{GaillardStEr14} can be obtained. 
Unfortunately, the function $\Phi(R+r, C+r^2)$ turns out to not always be piecewise-convex,
which breaks our analysis. 
Whether $d = 2$ gives a low-regret algorithm and how to analyze it
remain an open question. 

\section{Proof of Theorem \ref{thm:implications}}\label{app:implications}
\begin{proof}
For the first result, the key observation is $\u \cdot \C_T = R(\u) + 2\u \cdot \L_T$.
We only consider the case when $R(\u) \geq 0$ since otherwise the statement is trivial.
By the condition we thus have $R(\u)^2 \leq (R(\u) + 2\u \cdot \L_T)A(\u)$,
which by solving for $R(\u)$ gives 
\[ R(\u) \leq \tfrac{1}{2}(A(\u)+\sqrt{A(\u)^2+8(\u \cdot \L_T)A(\u)}) \leq \sqrt{2(\u \cdot \TL_T)A(\u)} + A(\u),\]
proving the bound we want.

For the second result, let $\E_t$ denote the expectation conditioning on all the randomness up to round $t$.
So by the condition, we have $\E_t[r_{t, i^*}] = \sum_{i=1}^N p_{t,i} \E_t[\ell_{t,i}-\ell_{t,i^*}] \geq \alpha(1-p_{t,i^*})$,
and thus $\E[R_{T, i^*}] \geq \alpha S$ where we define $S = \sum_{t=1}^T \E[1-p_{t,i^*}]$.
On the other hand, by convexity we also have $|r_{t, i^*}| \leq \sum_{i=1}^N p_{t,i}|\ell_{t,i} - \ell_{t,i^*}| \leq 1-p_{t,i^*}$
and thus $\E[R_{T, i^*}] \leq \E[\sqrt{A(\e_{i^*})C_{T, i^*}}] \leq \sqrt{A(\e_{i^*}) S}$ 
by the concavity of the square root function.
Combining the above two statements gives $S \leq \frac{A(\e_{i^*})}{\alpha^2}$,
and plugging this back shows $ \E[R_{T, i^*}] \leq  \frac{A(\e_{i^*})}{\alpha}$.
The high probability statement follows from the exact same argument of \citet{GaillardStEr14}
using a martingale concentration lemma.
\end{proof}

%\section{Pseudocode of {\ANHTV} and Its Fast Version}\label{app:ANHTV}
%\begin{theorem}\label{thm:fast_ANHTV}
%\end{theorem}
%\section{Adaptive Regret for General OCO Problem Using {\ANH}}\label{app:adaptive_OCO}

\section{Proof of Theorem \ref{thm:time_varying}}\label{app:time_varying}
Below we use $[s, t]$ to denote the set $\{s, s+1, \ldots, t-1, t\}$ for $1 \leq s \leq t \leq T$ and $s, t \in \mathbb{N}^+$.
\begin{proof}
We first fix an expert $i$ and consider the regret to this expert $\sum_{t=1}^T u_{t,i} r_{t,i}$. 
Let $a_j \geq 0$ and  $U_j = [s_j, t_j]$ for $j=1,\ldots, M$ be $M$ positive numbers and $M$
corresponding time intervals such that $u_{t,i} = \sum_{j=1}^M a_j \1\{t \in U_j\}$.
Note that this is always possible, with a trivial choice being $a_j = u_{t,j}$, $s_j = t_j = j$ and $M=T$;
we will however need a more sophisticated construction specified later.
By the adaptive regret guarantee, we have
\[ \sum_{t=1}^T u_{t,i} r_{t,i} = \sum_{j=1}^M a_j \sum_{t=1}^{T} \1\{t \in U_j\} r_{t,i}  
= \sum_{j=1}^M a_j R_{[s_j, t_j],i}  \leq \sum_{j=1}^M a_j \sqrt{A\sum_{t = s_j}^{t_j} z_{t,i}}, \]
which, by the Cauchy-Schwarz inequality, is at most
\[ \sqrt{\sum_{j=1}^M a_j} \cdot \sqrt{A \sum_{j=1}^M a_j \sum_{t = s_j}^{t_j} z_{t,i}} = 
\sqrt{\sum_{j=1}^M a_j} \cdot \sqrt{A \sum_{j=1}^M a_j \sum_{t=1}^{T} \1\{t \in U_j\} z_{t,i} } = 
\sqrt{\sum_{j=1}^M a_j} \cdot \sqrt{A \sum_{t=1}^T u_{t,i} z_{t,i}}. \]
Therefore, we need a construction of $a_j$ and $U_j$ such that $\sum_{j=1}^M a_j$ is minimized.
This is addressed in Lemma~\ref{lem:construction} below
which shows that there is in fact always a (optimal) construction such that 
$\sum_{j=1}^M a_j$ is exactly $\sum_{t=1}^T [u_{t,i} - u_{t-1,i}]_+$.
Now summing the resulting bound over all experts and applying the Cauchy-Schwarz inequality again proves the theorem.
\end{proof}

\begin{lemma}\label{lem:construction}
Let $v_1, \ldots, v_T$ be $T$ nonnegative numbers and $h(\{v_1, \ldots, v_T\}) = \min \sum_{j=1}^M a_j$
where the minimum is taken over the set of all possible choices of $M \in \mathbb{N}^+, a_j > 0, 
U_j = [s_j, t_j]$ with $1 \leq s_j \leq t_j \leq T, s_j, t_j \in \mathbb{N}^+ 
\;(j=1,\ldots, M)$ such that $v_t = \sum_{j=1}^M a_j \1\{t \in U_j\}$ for all $t$.
Then with $v_0$ defined to be $0$ we have
\[ h(\{v_1, \ldots, v_T\}) = \sum_{t=1}^T [v_{t} - v_{t-1}]_+. \]
\end{lemma}
\begin{proof}
We prove the lemma by induction on $T$. 
The base case $T=1$ is trivial. 
Now assume the lemma is true for any number of $v$'s smaller than $T$.
Suppose $a_j, U_j = [s_j, t_j] \;(j=1,\ldots, M)$ is an optimal construction.
Let $t^* \in \arg\min_t v_t$.
Without loss of generality assume $t^*$ belongs and only belongs to $U_1, \ldots, U_k$ for some $k$. 
By the definition of $h$, we must have
\[ h(\{v_1, \ldots, v_T\}) = v_{t^*} + h(\{v_1 - w_1, \ldots, v_{t^*-1} - w_{t^*-1}\}) 
+ h(\{v_{t^*+1} - w_{t^*+1}, \ldots, v_T - w_T\}), \]
where we define $w_t = \sum_{j=1}^k a_j \1\{t \in U_j\}$ and $h(\emptyset) = 0$ .
(Note that we also use the fact that $v_{t^*} = \min_t v_t$ so that $v_t - w_t \geq v_t - v_{t^*}$ is always nonnegative as required.)
Now the key idea is to show that ``extending'' each $U_j \;(j\in [k])$ to the entire time interval $[1,T]$ 
does not increase the objective value in some sense.
First consider extending $U_1$. By the inductive assumption, we have
\begin{align*}
&h(\{v_1 - w_1 - a_1, \ldots, v_{s_1-1} - w_{s_1-1} - a_1, v_{s_1} - w_{s_1}, \ldots, v_{t^*-1} - w_{t^*-1}\} \\
=\;& (v_1 - w_1 - a_1) + \(\sum_{t=2}^{s_1 - 1} [(v_{t}-w_{t} - a_1) - (v_{t-1}-w_{t-1} - a_1)]_+\) \\
\quad& + [(v_{s_1} - w_{s_1}) - (v_{s_1-1} - w_{s_1-1} - a_1)]_+ + \(\sum_{t=s_1+1}^{T} [(v_{t}-w_{t}) - (v_{t-1}-w_{t-1})]_+\)  \\
=\;& h(\{v_1 - w_1, \ldots, v_{t^*-1} - w_{t^*-1}\}) + [(v_{s_1} - w_{s_1}) - (v_{s_1-1} - w_{s_1-1} - a_1)]_+ \\
\quad& - ([(v_{s_1} - w_{s_1}) - (v_{s_1-1} - w_{s_1-1})]_+ + a_1) \\
\leq\;& h(\{v_1 - w_1, \ldots, v_{t^*-1} - w_{t^*-1}\}),
\end{align*}
where the inequality follows from the fact $[b+c]_+ \leq [b]_+ + c$.
Similarly, we also have
\begin{align*}
&h(\{v_{t^*+1} - w_{t^*+1}, \ldots, v_{t_1} - w_{t_1} , v_{t_1+1} - w_{t_1+1} - a_1, \ldots, v_T - w_T - a_1\} \\
\leq\;& h(\{v_{t^*+1} - w_{t^*+1}, \ldots, v_T - w_T\}).
\end{align*}
By extending $U_2, \ldots, U_k$ one by one in a similar way, we arrive at
\[ h(\{v_1, \ldots, v_T\}) \geq v_{t^*} + h(\{v_1 - v_{t^*}, \ldots, v_{t^*-1} - v_{t^*}\}) 
+ h(\{v_{t^*+1} - v_{t^*}, \ldots, v_T - v_{t^*}\}). \]
Notice that the right hand side of the above inequality admits a valid construction for $\{v_1, \ldots, v_T\}$:
an interval $U = [1, T]$ with weight $a = v_{t^*}$, together with the optimal constructions for
$\{v_1 - v_{t^*}, \ldots, v_{t^*-1} - v_{t^*}\}$ and $\{v_{t^*+1} - v_{t^*}, \ldots, v_T - v_{t^*}\}$.
Therefore, by the optimality of $h$, the above inequality must be an equality.
By using the inductive assumption again, we thus have
\begin{align*}
h(\{v_1, \ldots, v_T\}) &= v_{t^*} + h(\{v_1 - v_{t^*}, \ldots, v_{t^*-1} - v_{t^*}\}) 
+ h(\{v_{t^*+1} - v_{t^*}, \ldots, v_T - v_{t^*}\}) \\
&= v_{t^*} + \(v_1 - v_{t^*} + \sum_{t=2}^{t^*-1} [v_t - v_{t-1}]_+ \) + \(v_{t^*+1} - v_{t^*} + \sum_{t=t^*+2}^T [v_t - v_{t-1}]_+\) \\
&= \(\sum_{t=1}^{t^*-1} [v_t - v_{t-1}]_+\) + \(\sum_{t=t^*+1}^T [v_t - v_{t-1}]_+\) \\
&= \sum_{t=1}^T [v_{t} - v_{t-1}]_+,
\end{align*}
where the last step follows from $[v_{t^*} - v_{t^*-1} ]_+ = 0$ by the definition of $t^*$.
This completes the proof.
\end{proof}

\end{document}